\documentclass[12pt, article]{colt2022}

\usepackage[utf8]{inputenc}
\usepackage{amsmath}
\usepackage{amsfonts}
\usepackage{graphicx}
\usepackage{mathtools}
\usepackage{pst-all}
\usepackage{pgfplots}
\graphicspath{ {./Images/} }
\usepackage{comment}
\usepackage[ruled,algo2e]{algorithm2e}
\usepackage{hyperref}
\usepackage{nameref}
\usepackage[backend=biber,style=numeric,]{biblatex}
\DeclarePairedDelimiter\abs{\lvert}{\rvert}%
\DeclarePairedDelimiter\norm{\lVert}{\rVert}%
\makeatletter
\def\BState{\State\hskip-\ALG@thistlm}
\makeatother
\makeatletter
\let\oldabs\abs
\def\abs{\@ifstar{\oldabs}{\oldabs*}}
\let\oldnorm\norm
\def\norm{\@ifstar{\oldnorm}{\oldnorm*}}
\makeatother

\title[Limit Cycles of AdaBoost]{Limit Cycles of AdaBoost}
\usepackage{times}

\coltauthor{%
 \Name{Conor Snedeker} \Email{snedeker.31@buckeyemail.osu.edu}\\
 \addr Ohio State University \\
 Computer Science \& Engineering Department 
}

\begin{document}
\maketitle

\begin{abstract}
\noindent The iterative weight update for the AdaBoost machine learning algorithm may be realized as a dynamical map on a probability simplex. When learning a low-dimensional data set this algorithm has a tendency towards cycling behavior, which is the topic of this paper. AdaBoost's cycling behavior lends itself to direct computational methods that are ineffective in the general, non-cycling case of the algorithm. From these computational properties we give a concrete correspondence between AdaBoost's cycling behavior and continued fractions dynamics. Then we explore the results of this correspondence to expound on how the algorithm comes to be in this periodic state at all. What we intend for this work is to be a novel and self-contained explanation for the cycling dynamics of this machine learning algorithm.
\end{abstract}



\section{Introduction}
The AdaBoost algorithm is a popular machine learning implementation for binary classification learning presented by Freund and Schapire in 1997 \cite{adaboost} with a substantial literature of explanatory texts \cite{schap} \cite{boosting}. This algorithm iteratively reduces its classification error on a vast array of data sets, and may be adapted to more than binary classification problems \cite{hastie} \cite{multiclass}. AdaBoost is known as an ensemble method and closely resembles the canonical ensemble of statistical physics \cite[Ch.\ 19-20]{phys}. Ensemble methods allow for the creation of relatively complex and accurate classifiers from \emph{weak classifiers}.

Cycling dynamics in AdaBoost have been observed as early as Rudin, Daubechies, and Schapire in 2004 \cite{rudin}, where our paper gets much of its inspiration. In 2004 these researchers attempted to associate with these cycling dynamics the \emph{margins explanation} \cite{ogmargin} of \emph{Optimal} AdaBoost's effectiveness. In this context the margin of a data point of the data set AdaBoost learns is its distance from the decision boundary chosen at a particular iteration. 

This explanation for the effectiveness of the algorithm has been influential in guiding research towards variations and improvements on the AdaBoost \cite{marginvs} \cite{softmargin} \cite{rudin2007} \cite{wang}. The ultimate truth of this explanation is complicated \cite{breiman} \cite{reyzin}. However, the cycling dynamics of AdaBoost have been historically associated with the margins explanation given that the cycles may correspond to certain \emph{maximum margin} solutions for weak classifier choice \cite{ortiz} \cite{rudin}. Besides this margin explanation, AdaBoost's dynamics are also associated with a modest list of open problems \cite{bel} \cite{cycle}.

In this paper we present a correspondence between objects of classical dynamical systems and the algorithmic cycling behavior of AdaBoost. This relationship manifests in the iterative update of parameters by the algorithm, inducing actions equivalent to continued fractions dynamics on the unit interval. Rudin, Daubechies, and Schapire \cite{rudin} along with Belanich and Ortiz in 2012 \cite{ortiz} were able to implement sophisticated dynamical methods that help explain more general phenomena of the algorithm. However, our work is more precisely about what these cycles are with respect to the basic dynamics that induce them.

Our approach for the results of this paper come from modelling the cycling dynamics of AdaBoost in analogy with continued fractions dynamics. We begin in this direction by introducing AdaBoost along with the functions used to model its cycling behavior. Following this we introduce an \emph{edge update} function in Theorem~\ref{thm:3wgt} that we will use in analogy with a continued fractions map after some simplifications. The final result will be Theorem~\ref{rem:rat}, which proves the uniqueness of $k$-cycles induced by this continued fractions map in AdaBoost for each $k\in\mathbb{N}$.

\section{Preliminaries}


The AdaBoost algorithm makes use of an $n$-component, normalized \textbf{weight vector} $\vec{w}_t$ at iteration $t\in\mathbb{N}$, the \emph{edge} value that will be defined in the next section, and the abstract notion of a \emph{hypothesis space} populated with the possible ways in which the algorithm may classify the data set. In this sense these hypotheses are our weak learners that are combined to create a \emph{strong} learner.

AdaBoost iteratively applies hypotheses from its hypothesis space $\mathcal{H}$ and learns from its mistakes, focusing most on its misclassifications. Suppose that $m\in\mathbb{N}$. Let $\mathcal{D}$ be a data set with $n$ binary-labeled data points with $m$ features each. Our data set is represented by
\[\mathcal{D}=X\times Y\]
with features $X$ taken from $\mathcal{X}$ the set of possible features and $Y$ the binary labels for each point taken from $\mathcal{Y}$ the set of possible labels so that $X\times Y\subset\mathcal{X}\times\mathcal{Y}$. Conventionally, the labels in $Y$ take values $\pm 1$.

We initialize the weight vector
\[\vec{w}_1=( w_1,...,w_n)\]
so that $w_i=\frac{1}{n}$ for $i=1,...,n$. Each element of $\mathcal{D}$ is given a weight according to its significance to the algorithm. Initially, these weights are uniform. With each iteration, AdaBoost selects a hypothesis $h_t\in\mathcal{H}$ at iteration $t$ that maximizes the \textbf{edge}
value $r_t$ given below by our first definition.

For use in the work to come we will have to subdivide the weight component indexing set. From now on we define $i\in I$ to be exclusively the index for properly classified data points and $j\in J$ the misclassified data points, all for $t$ some iteration. Since iteration notation is present on all terms using index sets $I,J$, we omit iteration subscripts on $i,j$.
\begin{definition}[edge]
Take AdaBoost to be at iteration $t$. Suppose that $\vec{w}_t$ is an $n$-component weight vector and let $I\cup J$ be an index of $1,...,n$. Let $I,J$ be as defined in the previous paragraph according to the hypothesis $h_t$. Then we define the \textbf{edge} of AdaBoost to be the following value at iteration $t$
\begin{equation}\label{eq:edgeformula}
r_t=\sum_{i\in I}w_{i,t}-\sum_{j\in J}w_{j,t}.
\end{equation}
\end{definition}

In the above sense, a chosen hypothesis $h_t$ induces an edge at each iteration $t$ by AdaBoost selecting a hypothesis that maximizes the edge value. We must require that $r_t>0$ for all $t$, or else our weak learners classify at worse than random chance.

\newpage

\begin{definition}[mistake dichotomy \cite{ortiz}]\label{eq:mistake}
Let $\vec{w}_t$ be the weight vector of AdaBoost with $n$-components at iteration $t$. The algorithm must track properly classified and misclassified data points at $t$ for use in the edge definition. Using the same index as $\vec{w}_t$ we define an $n$-component vector 
\begin{equation*}
\eta_t=( \eta_{1,t},...,\eta_{n,t})
\end{equation*}
so that $\eta_{i,t}\in\{\pm 1\}$ with $\eta_{i,t}$ matching the sign of $w_{i,t}$ in the edge  formula. We call this vector the \textbf{mistake dichotomy}. If we let $\vec{y}=(y_i)^{n}_{i=1}$ be the vector of labels from $\mathcal{D}$, we can also define the mistake dichotomy at iteration $t$ with $\eta_t=(y_ih_t(x_i))^{n}_{i=1}$ for $h_t(x_i)$ the value of the $i$-th data point hypothesized by $h_t$.
\end{definition}

\begin{algorithm}
 \caption{\emph{Optimal} AdaBoost}\label{alg:adaboost}
 \SetAlgoLined
 \KwData{$\mathcal{D}$}
 \KwResult{Combined classifier}
 initialization\;
 
 $w_{i,0} \gets \frac{1}{n}$ \text{\ \ \ \ \ \ \ \ \ \ \ \ \ \ \ \ \ \ \ \ \ \ \ \ \ \ \ \ \ \ \ \ \ \ \ \ \ \ \ \ \ \ \ \ \ \ for $n$ components of $\vec{w}_0$}
 
 $t \gets 0$\;
 
 \For{$t\leq t_{max}$}{
 
  $\eta_t\in\text{argmax}_{\eta'}(\vec{w}_t\cdot\eta')$\;
  
  $r_t=\vec{w}_t\cdot\eta_t$ \text{\ \ \ \ \ \ \ \ \ \ \ \ \ \ \ \ \ \ \ \ \ \ \ \ \ \ \ \ \ \ \ \ \ \ \ \ the optimal edge at iter. $t$}\;
  
  $\alpha_t=\frac{1}{2}\log\left(\frac{1+r_t}{1-r_t}\right)$\;
  
  $H_t(x)=H_{t-1}(x)+\alpha_th_{t}(x)$ \text{\ \ \ \ \ \ \ \ update combined classifier}\;
  
  $w_{i,t}\gets w_{i,t}e^{-\eta_{i,t}\alpha_t}$\;
  
  $w_{i,t+1}=\frac{w_{i,t}e^{-\eta_{i,t}\alpha_t}}{Z_{t}}$ \text{\ \ \ \ \ \ \ \ \ \ \ \ \ \ \ \ \ \ \ \ \ \ normalization for each $i$}\;
  
}
\textbf{return} $H_{t_{max}}(x)$ \text{\ \ \ \ \ \ \ \ \ \ \ \ \ \ \ \ \ \ \ \ \ \ \ \ \ \ \ \ \ \ \ \ \ final classifier}

\end{algorithm}

Where we define the \emph{normalization constant} $Z_{t}$ to be
\[Z_{t}=\sum^n_{i=1}w_{i,t}e^{-\eta_{i,t}\alpha_t}.\]

Via \cite{rudin} we can use an alternative update formula instead of the computationally cumbersome exponential form that is given in the algorithm pseudocode.
\begin{definition}[weight update]
Given a vector $\vec{w}_t$ of $n$ weights at iteration $t$ along with its mistake dichotomy of the same iteration, we may use iteration $t$ weights to calculate iteration $t+1$ weights by
\begin{equation}\label{eq:wgtformula}
w_{i,t+1}=\frac{w_{i,t}}{1+ \eta_{i,t}r_t}.
\end{equation}
We call this the \textbf{weight update} from iteration $t$ to $t+1$.
\end{definition}

This definition will factor prominently into both the motivations and mathematical work to come. The alternative weight update above greatly simplifies the computations associated with analyzing the structure of our AdaBoost algorithm. As we will see, this structure is surprisingly simple when the algorithm converges to a closed loop. It is this alternative weight update that makes the correspondence to continued fractions dynamics most obvious.

An important relation to note is an alternative writing of the edge formula Eqn.~\ref{eq:edgeformula} using the fact that for all $t$ we have $\sum_{i\in I\cup J}w_{i,t}=1$, which implies
\begin{equation}\label{eq:edgealt}
r_{t}=1-2\sum_{j\in J}w_{j,t}
\end{equation}
for $J$ the index of misclassified data points at iteration $t$ with respect to a chosen hypothesis that maximizes $r_t$.

\begin{definition}[mistake lattice]\label{def:mistake lattice}
Taking the transpose of the row-vector form of mistake dichotomies for iterations of at least $t$, we can write out all of these dichotomies in a row
\[
\begin{pmatrix}
\eta_{1,t}&\eta_{1,t+1}&\eta_{1,t+2}&\dots\\
\eta_{2,t}&\eta_{2,t+1}&\eta_{2,t+2}&\dots\\
\eta_{3,t}&\eta_{3,t+1}&\eta_{3,t+2}&\dots\\
\vdots&\vdots&\vdots&\\
\eta_{n,t}&\eta_{n,t+1}&\eta_{n,t+2}&\dots
\end{pmatrix}.
\] 
We call this the \textbf{mistake lattice}, or simply \textbf{lattice}, of AdaBoost's mistakes from iteration $t$ onward. Furthermore, we will refer to the lattice as having a correspondence $w_{i,t}\sim \eta_{i,t}$ that handily represents the labelling success of AdaBoost on the $i$-th data point at iteration $t$.
\end{definition}

\begin{remark}[\cite{prum}]
This naming convention comes from statistical mechanics, in which the representation of binary states in the Ising Model is referred to as a lattice. AdaBoost forms a 2D lattice in the sense of the above mistake lattice via iterations over data points of $\mathcal{D}$. The 2D mistake lattice of AdaBoost is a record of the algorithm's mistakes rather than a 2D grid of interacting particles, unlike the Ising Model.
\end{remark}

\begin{definition}
Suppose that we have any lattice as defined above that corresponds to $n$ data points starting at iteration $t$. If this lattice of AdaBoost has the condition that $\eta_{i,j} = -1$ implies $\eta_{i,j+1} = 1$ starting at iteration $t$, we say it has the \textbf{periodic learning condition}, or $\nabla$.
\end{definition}

\begin{remark}
Although this condition is rather strong and unusual in the context of machine learning, the ordering of lattices in this way is common in statistical mechanics \cite{order0} and material sciences \cite{order1}. In a particularly similar venue, we see these sorts of ordered structures in lattices of the Ising Model with respect to magnetism \cite{order2}. The Ising Model is mathematically related to AdaBoost via the formalism of canonical ensembles \cite{prum}. AdaBoost is an iteratively defined 1D canonical ensemble with no particle interactions, but rather forward interactions between iterations of the 1D lattice structure that form the combined 2D lattice above. Carefully, we cannot conflate the iterative 2D mistake lattice with a 2D Ising Model lattice as a 2D canonical ensemble is not iterative in the same way as AdaBoost.
\end{remark}

\newpage

\begin{example}
A simple example of this condition is the following
\[
\begin{pmatrix}
1&1&-1&1&1&-1&\dots\\
1&-1&1&1&-1&1&\dots\\
-1&1&1&-1&1&1&\dots\\
\end{pmatrix}.
\] 
This mistake lattice is the least complex possible example. The lattice arises when AdaBoost has access only to the three dichotomies $(-1,1,1),(1,-1,1),(1,1,-1)$. In \cite{rudin}, the algorithm is shown to cycle when it has access to just these $3$ mistake dichotomies.
\end{example}

While the periodic learning condition may seem artificial, we have been able to find it in data sets of several hundred unique data points. After Section~4 we will have the tools needed to show its natural occurrence in data sets such as the UCI \emph{Iris} data set among others. A collection of empirical results confirming the occurrence of this condition is found in Appendix~\ref{sec:appB} after our formal proofs appendix.

\section{A first look at continued fractions maps}

Our overall motivations for the techniques in this paper come from some simple observations about fixed points for continued fractions maps. This section is written to introduce a reader to these maps and how they relate to simple cases of AdaBoost's cycling dynamics.

In Rudin, Daubechies, and Schapire \cite{rudin} we are presented with a number of experimental, often geometric and computational explanations for what an AdaBoost cycle appears to be in relation to Algorithm~\ref{alg:adaboost}. One may even wonder at this point how we will pull a classical dynamical system out of this algorithm definition. However, there is something quite simple and nice happening in these cycles that can be calculated directly. 

In order to motivate this coming result, consider an example in \cite{rudin} of a cycle that eventually has the fixed edge $r_t=\frac{\sqrt{5}-1}{2}$ for all sufficiently large $t$ and a $3$-cycle on its $3$ weights. A demonstration of a similar cycle in the UCI repository \emph{Iris} data set can be found at the beginning of Appendix~\ref{sec:appB}.

\begin{proposition}
 We find that a function which has a fixed point of this value, the Golden Ratio minus $1$, is the Gauss map \emph{\cite[Ch.\ 1]{brin}}. This continued fractions iterated function is defined
\[G(x)=\frac{1}{x}\mod_1\]
for $x\in (0,1)$.
\end{proposition}
\begin{proof}
We may verify this claim via
$
G\left(\frac{\sqrt{5}-1}{2}\right)=\frac{2}{\sqrt{5}-1}\mod_1=\frac{\sqrt{5}-1}{2}
$
where the rightmost equality follows from the fact that $1<\frac{2}{\sqrt{5}-1}<2$ and $\frac{2}{\sqrt{5}-1}-1=\frac{3-\sqrt{5}}{\sqrt{5}-1}\frac{\sqrt{5}+1}{\sqrt{5}+1}$.
\end{proof}

 However, this is not the full story given that the Gauss map will not induce the necessary cycles. There is another simple dynamical function that is integral to the coming work. Indeed, for $x\in [\frac{1}{2},1]$, it is identical to the Gauss map.

\begin{definition}[Farey map \cite{farey}]
We define the \textbf{Farey map}, which carries much of the concrete information on AdaBoost's periodic behavior, as
\begin{equation*}
F(x)=\begin{cases}
      \frac{x}{1-x} &\text{if } 0\leq x < \frac{1}{2} \\
      \frac{1-x}{x} &\text{if } \frac{1}{2}\leq x \leq 1 
   \end{cases}
\end{equation*}
that also has $\frac{\sqrt{5}-1}{2}$ as a fixed point.
\end{definition}

The Farey map has the convenient property that it cycles on \emph{quadratic irrationals}. Many of the simpler cycles we see in \cite{rudin} are actually just sequences of quadratic irrational approximations in practice and can be easily replicated. While we do not claim \emph{all} AdaBoost cycles are induced by the Farey map, the simplest examples appear to be so.

It is not difficult to create more cycles as examples using the AdaBoost algorithm itself. 
We can go about finding either a matrix as in the literature \cite{ortiz}, \cite{rudin} or a particular data set that induces cycling behavior and adjust its selection process for the edge value at each iteration. Optimal AdaBoost selects a hypothesis that induces the largest edge value at any iteration. 

After finding an input that induces cycling from AdaBoost, we can select a value less than the maximum edge value available and run the algorithm again. This time, if we have AdaBoost select the first edge above $\frac{2}{5}$, the algorithm should converge to the $2$-cycle on the edge values 
$\frac{1}{\sqrt{2}}, \sqrt{2}-1$.

These are again induced by the Farey map given this straightforward computation
\[
\frac{1-\frac{1}{\sqrt{2}}}{\frac{1}{\sqrt{2}}}=\sqrt{2}\left(1-\frac{1}{\sqrt{2}}\right) \text{ and } \frac{\sqrt{2}-1}{1-(\sqrt{2}-1)}=\frac{1}{\sqrt{2}}\frac{\sqrt{2}-1}{\sqrt{2}-1}.
\]
This is no coincidence, though we must make some modifications to be precise moving forward. For technical reasons what we seek is the inverse of the Farey map. We must do this to ensure that our dynamics have the right fixed point, which ought to be $\frac{\sqrt{5}-1}{2}$ given the algorithm's tendency to find this edge value.

\begin{definition}[Farey map inverses]
We will be interested in the inverse functions to the Farey map on its two intervals, which we will call the \textbf{Farey map inverses}. These are
\begin{equation}\label{eqn:inverses}
    L(x)=\text{\emph{inv}}\left(\frac{x}{1-x}\right)=\frac{x}{x+1}\text{ and }R(x)=\text{\emph{inv}}\left(\frac{1-x}{x}\right)=\frac{1}{x+1}
\end{equation}
for \emph{inv}$(f(x))$ the inverse function of the function $f(x)$.
\end{definition}

It is the $R$ function that induces convergence to a cycle, as it has fixed point $\frac{\sqrt{5}-1}{2}$ that is also attracting on $[0,1]$ since $\abs{\frac{d}{dx}R(x)}=\frac{1}{(1+x)^2}<1$ for all $x\in (0,1]$ \cite[Ch.\ 1]{brin}. The $L$ function also has a fixed point that attracts, but it is at $0$. The case of the edge converging to $0$ is the \emph{unrealizable} case of AdaBoost \cite{rudin} and will never correspond to non-trivial dynamics.

\begin{figure}[htp]
    \centering
    \includegraphics[width=11cm]{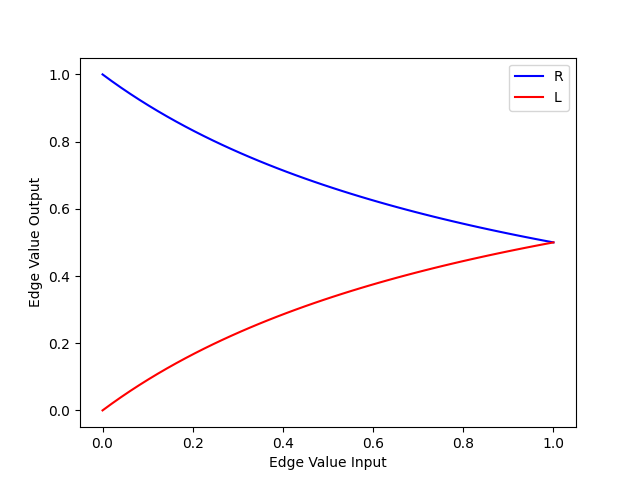}
\end{figure}

What is most interesting about this preliminary work involving the functions $L$ and $R$ is the method of attraction towards a Farey map cycle. By definition, taking $s\in\mathbb{N}$ and the composition $R^{(s)}$ we have that
\[\lim_{s\to\infty}R^{(s)}(1)=\frac{\sqrt{5}-1}{2}.\]
This is actually true for any $x$ for $x\in [0,1]$. So, convergence to one of these Farey map cycles appears to be straightforward. If the edge formula equals $R$ repeatedly, then it will iteratively evaluate closer and closer to the fixed point $\frac{\sqrt{5}-1}{2}$ until it reaches a value from which it begins to cycle in earnest. In the case of our fixed point edge cycle, this is simply the same value to which $R^{(s)}$ converges.

As shown in \cite{rudin} if AdaBoost is running on just three data points and three unique mistake dichotomies, then we can easily write out its dynamics by hand. However, dealing with these cycles in full generality will require some of our formalism from the coming two sections.

Suppose that $c\in\mathbb{N}$. Let $\{r_t,r_{t+1},...\}$ be a sequence of edges of AdaBoost starting at iteration $t$. If this sequence is defined as a cycle on the Farey map inverses, then for some $k\in\mathbb{N}$ we have that $r_t=r_{t+ck}$ for any $c$. Furthermore, we assume that $r_{t+i}\neq r_{t+j}$ for $i\neq j$ given $0\leq i,j\leq k-1$. We call this a \textbf{cycle on the Farey map inverses} and we will often refer to it as a \textbf{cycle on the Farey map}.

\section{How AdaBoost induces edge cycles on the Farey map}

In this section we will demonstrate the conditions under which AdaBoost's edge formula takes on a form that corresponds to continued fractions dynamics. Through our research this seems to be the basis for most or all cycling dynamics of the algorithm.

We will soon need to deal in more than just the $2$ indices defined in the edge formula of Eqn.~\ref{eq:edgeformula}. With some modifications this edge definition can be used to write the edge $r_t$ in terms of the weights of iteration $t-1$. However, in order to do anything like this, we must also partition our indices to include information about the mistake dichotomy of this previous iteration. We do this by defining $4$ index sets $I^+,I^-,J^+,J^-$. If $i\in I$ at iteration $t$ then either $i\in I^+$ or $i\in I^-$ depending on if $w_{i,t}=w_{i,t-1}\frac{1}{1+r_{t-1}}$ or $w_{i,t}=w_{i,t-1}\frac{1}{1-r_{t-1}}$ respectively, and we do the same with $j\in J$ using $J^+$ and $J^-$. This is convenient for notation since $I^+\cup I^-=I$, but it also allows for the following.

\begin{proposition}\label{prop:fourterm}
Let $\vec{w}_t$ be an $n$-component weight vector of AdaBoost at iteration $t$. The edge definition in Eqn.~\ref{eq:edgeformula} for $t$ can be written as a four term sum namely
\begin{align*}\label{eq:inverse}
\begin{split}
r_t = &\sum_{i^+\in I^+}w_{i^+,t-1}\frac{1}{1+r_{t-1}}-\sum_{j^+\in J^+}w_{j^+,t-1}\frac{1}{1+r_{t-1}}\\
 & +\sum_{i^-\in I^-}w_{i^-,t-1}\frac{1}{1-r_{t-1}}-\sum_{j^-\in J^-}w_{j^-,t-1}\frac{1}{1-r_{t-1}}
\end{split}
\end{align*}
with $I^+,I^-,J^+,J^-$ as defined immediately before this proposition. 
\end{proposition}
\begin{proof} Using the weight update formula from Eqn.~\ref{eq:wgtformula} we can write the edge formula as
\begin{equation}
r_t=\sum_{i\in I\cup J}\eta_{i,t}w_{i,t-1}\frac{1}{1+\eta_{i,t-1}r_{t-1}}
\end{equation}
for $\eta_{i,t}$ the values of the mistake dichotomy at component $i$ for iteration $t$. Working out how this partitions each set of positive and negative weights in the linear sum we get the result.
\end{proof}

Using this four term edge calculation, we can define a useful definition that will be mentioned many times throughout this paper. In Eqn.~\ref{eq:edgeformula} we defined the edge as a value at a specific iteration, and the following is an \emph{edge update} from one iteration to another.

\begin{definition}
Suppose we have the conditions of Proposition~\ref{prop:fourterm}. When we are given the edge value $r_t$ in terms of weights and the edge of iteration $t-1$, we call this the \textbf{edge update} from iteration $t-1$ to $t$.
\end{definition}

\begin{theorem}[\hyperlink{10}{Proof}]\label{thm:3wgt}
Let $\vec{w}_t$ be an $n$-component weight vector at iteration $t$. If we assume that the periodic learning condition holds for all iterations, then the AdaBoost edge update simplifies to
\begin{equation}\label{eqn:edgeupdate}
    r_t=\frac{1+r_{t-1}-2\sum_{j^+\in J^+}w_{j^+,t-1}}{1+r_{t-1}}.
\end{equation}
Also, the converse statement is true.
\end{theorem}

This theorem is crucial for understanding the cycling dynamics of AdaBoost. It sets up a criterion for the sequential selection of mistake dichotomies to induce a nice edge update form that resembles a continued fractions dynamical map. It seems likely that this sort of criterion governs \emph{most} cycling behavior for AdaBoost, although we will not present proof for that claim in this work.

As stated previously in Section~2, we now have the ability to show the natural occurrence of the periodic learning condition in publicly available data sets. Since Theorem~\ref{thm:3wgt} is an \emph{if and only if} statement, $\sum_{j^+\in J^+}w_{j^+,t-1}=\frac{r_{t-1}}{2}$ implies that Eqn.~\ref{eqn:edgeupdate} evaluates to $\frac{1}{1+r_{t-1}}$ under the periodic learning. The Golden Ratio minus $1$ is the unique fixed point of $\frac{1}{1+r_{t-1}}$ and evidence of this fixed point edge cycle from various data sets can be found in Appendix~\ref{sec:appB}.

Indeed, it is the value of the sum 
$\sum_{j^+\in J^+}w_{j^+,t-1}$ that determines the form of the RHS of Eqn.~\ref{eqn:edgeupdate}. When $\sum_{j^+\in J^+}w_{j^+,t-1}=\frac{1}{2}$, we get $r_t=\frac{r_{t-1}}{1+r_{t-1}}$ and combining this with the case in the previous paragraph gives us the inverse Farey map values from Eqn.s~\ref{eqn:inverses}. Hence now we see how AdaBoost may come to cycle on its edge values given the periodic learning condition and appropriate values for the edge update parameters.

\section{Cycling relation between edges and weights}

Consider edge cycles that are not restricted to the Farey map inverses. As shown in Theorem~\ref{thm:3wgt}, when the data set has the periodic learning condition $\nabla$, this induces an edge update
\begin{equation*}
    r_t=\frac{1+r_{t-1}-2\sum_{j^+\in J^+}w_{j^+,t-1}}{1+r_{t-1}}
\end{equation*}
for iterations $t-1$ of which the condition holds.
While we do not know that each cycle must take place with this edge update form, it appears to us that $\nabla$ is related to all cycling behavior. We will now analyze this more general form, rather than merely the special case of Farey map inverses.

A nice property of AdaBoost under the $\nabla$ condition is that we may use the proof of Theorem~\ref{thm:3wgt} to obtain the combined value of each subsum given by
\begin{eqnarray*}
r_t=\sum_{i^+\in I^+}w_{i^+,t-1}\frac{1}{1+r_{t-1}}-\sum_{j^+\in J^+}w_{j^+,t-1}\frac{1}{1+r_{t-1}}
+\sum_{i^-\in I^-}w_{i^-,t-1}\frac{1}{1-r_{t-1}}.
\end{eqnarray*}

\begin{proposition}[\hyperlink{20}{Proof}]\label{prop:wvals}
When condition $\nabla$ holds the three subsums above have unique values in $\{\frac{1}{2},\frac{r_{t}}{2},\frac{1-r_{t}}{2}\}$.
\end{proposition}

This above proposition will be vital in getting a grasp on the cycling dynamics of AdaBoost. If we have an $n$-component weight vector $\vec{w}_t$ at iteration $t$, then we have to consider how AdaBoost distributes the combined values as above over its perhaps numerous weights. As in \cite{rudin} we can use simplex analysis to understand what happens to weight values as data sets grow larger. 

Each weight vector $\vec{w}_t$ exists inside of an $(n-1)$-dimensional probability simplex in $\mathbb{R}^n$ subdivided by decision boundaries whose $m\in\mathbb{N}$ interiors represent the algorithm's choice of hypothesis with an optimal mistake dichotomy. In experimentation it appears that when AdaBoost is cycling, weight values for specific cycles will distribute in nice ways.

For this reason it appears that AdaBoost may scale a very short cycle up to large data sets. One may receive a $3$-cycle on the weights of AdaBoost from a data set with a hundred data points. We will work towards applying these assumptions with the following notion. Suppose we index the weights of $\vec{w}_t$ with the set $I\cup J=\{1,...,n\}$. We move towards formalizing AdaBoost's tendency to nicely distribute weight values across components of the weight vector via a sequence $\{\beta_j\}_{j\in I\cup J}$ with $\beta_j\in\mathbb{Q}\cap [0,1]$ for all $j$ so that $\sum_j\beta_j\in\mathbb{N}$.

\begin{definition}[alternative weights]\label{def:altw}
Let $i\in I\cup J$ and define $\beta_i\in[0,1]$ for all $i$. For AdaBoost cycling with condition $\nabla$, we subdivide the indexing set of an $n$-component weight vector $\vec{w}_t$ by
$I\cup J=I^+\cup I^-\cup J^+$
as in our previous work from earlier sections.

Given Proposition~\ref{prop:wvals} and Proposition~\ref{prop:fourterm} combined with $\nabla$ we have
$r_t=\frac{r_t}{2}+\frac{1}{2}-\frac{1-r_t}{2}$
so that each subsum of weights has individual contributions indexed by $I^+,I^-,J^+$ with
$\{\beta_{i^+,t}\}_{i^+\in I^+}$, $\{\beta_{i^-,t}\}_{i^-\in I^-}$, $\{\beta_{j^+,t}\}_{j^+\in J^+}$. To make our coming calculations work, the sum of $\beta_{i,t}$ terms over individual subindex sets $I^+,I^-,J^+$ is $1$. We call $\beta_{i,t}$ the \textbf{contribution} of weight $i\in I\cup J$ at iteration $t$.
This gives us \textbf{alternative weight} representations for $\vec{w}_t$ that sum
\[r_t=\sum_{i^+\in I^+}\beta_{i^+,t}\frac{r_t}{2}+\sum_{i^-\in I^-}\beta_{i^-,t}\frac{1}{2}-\sum_{j^+\in J^+}\beta_{j^+,t}\frac{1-r_t}{2}\]
where for all $i\in I\cup J$ we have $w_{i,t}\in\{\beta_{i,t}\frac{1}{2},\beta_{i,t}\frac{r_t}{2},\beta_{i,t}\frac{1-r_t}{2}\}$ depending on the membership of $i$ in the subindices $I^+,I^-,J^+$.
\end{definition}
\begin{remark}\label{rem:rat}
As noted before the weights of AdaBoost are geometrically the vectors in an $(n-1)$-dimensional probability simplex with $n$-component distributions. Adding dimensions to the simplex to accommodate larger data sets preserves the geometric relations of $(\beta_{i,t})_{i\in I}$ in the simplex as a subset of $\mathbb{R}^n$ up to scaling by a positive value less than $1$. The individual values of $\beta_i$ generally behave well and we assume for use in our mathematics that $\beta_i\in\mathbb{Q}\cap [0,1]$ via our own experience with the algorithm, and this is reflected in some calculations seen in \cite{rudin}. Also, note that for not all $i\in I\cup J$ is $\beta_i=0$ true without obtaining trivial dynamics. 
\end{remark}
In experimentation AdaBoost has a tendency to redistribute the alternative weight contributions shown above. The algorithm may change the contributions to certain values, make a contribution $0$ entirely, and take a $0$-valued contribution to a positive value. Generally this does not change the cycle in any fundamental way except permuting or reweighting rows, and at times even reversing the cycle altogether \cite{rudin}.

We will not worry about this sort of redistribution behavior and treat $\beta_i$ as a fixed value. The tendency for AdaBoost to do this even after finding a limit cycle is most likely a matter of implementation, and not dynamically significant for us here. While these alternative weights may seem obvious it allows us to reason about the weights of AdaBoost using the set $\{\frac{1}{2},\frac{r_t}{2},\frac{1-r_t}{2}\}$ along with what we know about the Farey map inverses. 

\begin{theorem}[\hyperlink{23}{Proof}]\label{thm:gencyc}
Suppose that AdaBoost is cycling on the edge values $\{r_l,...,r_{l-1+k}\}$ and that $\nabla$ holds. Let $\vec{w}_t$ be an $n$-component weight vector at iteration $l\leq t$. Given these conditions, if two lattices of AdaBoost on iterations $t+sk$ to $t+ck$ for $s<c$ agree on their corresponding weights and mistake dichotomy at $t+j$, then the weights and mistake dichotomies agree for all iterations.
\end{theorem}

We end this section at the above result, given that any further details on what is and is not completely determined by the edge cycles can be worked out from here.

\section{Discussion}

What is most striking about this explanation for the cycling dynamics of AdaBoost is how the algorithm passes to these periodic states on particular data sets. This cycling behavior is most easily seen in the artificial matrix input form of \cite{rudin}, but also arises on data sets that are often low in complexity or low in data point count. Without any obvious pre-programming, the algorithm is able to utilize these especially simple and uniform routines to execute its optimizations. This cycling behavior is induced from only modest variation in the original, full breadth of evolutions available to the dynamical system in its non-cycling form.

These results may very well help to solve the open problem associated with the paper \cite{rudin}, which can be seen stated in this group's 2012 followup paper \cite{cycle}. This problem is whether each of these simulated input matrices $\mathbb{M}$ in place of data sets always induce cycling dynamics. The work showcased here came from attempting to answer this question. However, even with the details present it was not clear how to accomplish this goal.

Our results from the main body sections of this paper are quite particular to the case of cycling AdaBoost. Although we are not able to extend these methods to the general case of AdaBoost's dynamics, the dynamical systems chosen here are \emph{just right} for the task we set out to accomplish. Once this algorithm has converged to a limit cycle we are given dynamics that are predictable under the right lens, requiring no approximations or long-term recurrence properties. In this sense the project of this paper is a success.


\clearpage
\printbibliography[heading=bibintoc,title={References}]

@article{rudin,
  title={The dynamics of AdaBoost: cyclic behavior and convergence of margins.},
  author={Rudin, Cynthia and Daubechies, Ingrid and Schapire, Robert E and Ron, Dana},
  journal={Journal of Machine Learning Research},
  volume={5},
  number={10},
  year={2004}
}

@article{ortiz,
  title={On the convergence properties of optimal adaboost},
  author={Belanich, Joshua and Ortiz, Luis E},
  journal={arXiv preprint arXiv:1212.1108},
  year={2012}
}

@book{brin,
  title={Introduction to dynamical systems},
  author={Brin, Michael and Stuck, Garrett},
  year={2002},
  publisher={Cambridge university press}
}

@article{farey,
  title={Petrie polygons, Fibonacci sequences and Farey maps},
  author={Singerman, David and Strudwick, James},
  journal={Ars Mathematica Contemporanea},
  volume={10},
  number={2},
  pages={349--357},
  year={2016}
}

@online{multiclass,
    author = "Avinash Kak",
    title = "AdaBoost for Learning Binary and Multiclass
Discriminations",
    url  = "https://engineering.purdue.edu/kak/Tutorials/AdaBoost.pdf",
    addendum = "(accessed: 01.27.2022)"
}

@book{phys,
  title={Perspectives on Statistical Thermodynamics},
  author={Oono, Yoshitsugu},
  year={2017},
  publisher={Cambridge University Press}
}

@incollection{schap,
  title={Explaining adaboost},
  author={Schapire, Robert E},
  booktitle={Empirical inference},
  pages={37--52},
  year={2013},
  publisher={Springer}
}

@inproceedings{cycle,
  title={Open Problem: Does AdaBoost Always Cycle?},
  author={Rudin, Cynthia and Schapire, Robert E and Daubechies, Ingrid},
  booktitle={Conference on Learning Theory},
  pages={46--1},
  year={2012},
  organization={JMLR Workshop and Conference Proceedings}
}

@article{adaboost,
  title={A decision-theoretic generalization of on-line learning and an application to boosting},
  author={Freund, Yoav and Schapire, Robert E},
  journal={Journal of computer and system sciences},
  volume={55},
  number={1},
  pages={119--139},
  year={1997},
  publisher={Elsevier}
}

@inproceedings{marginvs,
  title={Boosting the minimum margin: LPBoost vs. AdaBoost},
  author={Li, Hanxi and Shen, Chunhua},
  booktitle={2008 Digital Image Computing: Techniques and Applications},
  pages={533--539},
  year={2008},
  organization={IEEE}
}

@article{softmargin,
  title={Soft margins for AdaBoost},
  author={R{\"a}tsch, Gunnar and Onoda, Takashi and M{\"u}ller, K-R},
  journal={Machine learning},
  volume={42},
  number={3},
  pages={287--320},
  year={2001},
  publisher={Springer}
}

@article{ogmargin,
  title={Boosting the margin: A new explanation for the effectiveness of voting methods},
  author={Bartlett, Peter and Freund, Yoav and Lee, Wee Sun and Schapire, Robert E},
  journal={The annals of statistics},
  volume={26},
  number={5},
  pages={1651--1686},
  year={1998},
  publisher={Institute of Mathematical Statistics}
}

@article{rudin2007,
  title={Analysis of boosting algorithms using the smooth margin function},
  author={Rudin, Cynthia and Schapire, Robert E and Daubechies, Ingrid},
  journal={The Annals of Statistics},
  volume={35},
  number={6},
  pages={2723--2768},
  year={2007},
  publisher={Institute of Mathematical Statistics}
}

@inproceedings{reyzin,
  title={How boosting the margin can also boost classifier complexity},
  author={Reyzin, Lev and Schapire, Robert E},
  booktitle={Proceedings of the 23rd international conference on Machine learning},
  pages={753--760},
  year={2006}
}

@article{breiman,
  title={Prediction games and arcing algorithms},
  author={Breiman, Leo},
  journal={Neural computation},
  volume={11},
  number={7},
  pages={1493--1517},
  year={1999},
  publisher={MIT Press One Rogers Street, Cambridge, MA 02142-1209, USA journals-info~…}
}

@book{boosting,
author = {Schapire, Robert E and Freund, Yoav},
title = {Boosting: Foundations and Algorithms},
year = {2012},
isbn = {0262017180},
publisher = {The MIT Press}
}

@article{wang,
  title={A refined margin analysis for boosting algorithms via equilibrium margin},
  author={Wang, Liwei and Sugiyama, Masashi and Jing, Zhaoxiang and Yang, Cheng and Zhou, Zhi-Hua and Feng, Jufu},
  journal={The Journal of Machine Learning Research},
  volume={12},
  pages={1835--1863},
  year={2011},
  publisher={JMLR. org}
}

@article{hastie,
  title={Multi-class adaboost},
  author={Hastie, Trevor and Rosset, Saharon and Zhu, Ji and Zou, Hui},
  journal={Statistics and its Interface},
  volume={2},
  number={3},
  pages={349--360},
  year={2009},
  publisher={International Press of Boston}
}

@article{bel,
  title={Some Open Problems in Optimal AdaBoost and Decision Stumps},
  author={Belanich, Joshua and Ortiz, Luis E},
  journal={arXiv preprint arXiv:1505.06999},
  year={2015}
}

@book{prum,
  title={Stochastic processes on a lattice and Gibbs measures},
  author={Prum, Bernard and Fort, Jean Claude},
  volume={11},
  year={1990},
  publisher={Springer Science \& Business Media}
}

@article{order0,
  title={Magnetic order in a frustrated two-dimensional atom lattice at a semiconductor surface},
  author={Li, Gang and H{\"o}pfner, Philipp and Sch{\"a}fer, J{\"o}rg and Blumenstein, Christian and Meyer, Sebastian and Bostwick, Aaron and Rotenberg, Eli and Claessen, Ralph and Hanke, Werner},
  journal={Nature communications},
  volume={4},
  number={1},
  pages={1--6},
  year={2013},
  publisher={Nature Publishing Group}
}

@article{order1,
  title={Magnetic ordering in a frustrated bow-tie lattice},
  author={Stimpson, Laura J Vera and Rodriguez, Efrain E and Brown, Craig M and Stenning, Gavin BG and Jura, Marek and Arnold, Donna C},
  journal={Journal of Materials Chemistry C},
  volume={6},
  number={16},
  pages={4541--4548},
  year={2018},
  publisher={Royal Society of Chemistry}
}

@article{order2,
  title={Order by disorder in the antiferromagnetic Ising model on an elastic triangular lattice},
  author={Shokef, Yair and Souslov, Anton and Lubensky, Tom C},
  journal={Proceedings of the National Academy of Sciences},
  volume={108},
  number={29},
  pages={11804--11809},
  year={2011},
  publisher={National Acad Sciences}
}

\appendix

\section*{APPENDIX}

\section{Proofs of mathematical results}

\hypertarget{10}{}
\begin{proof}{\bf of Theorem~\ref{thm:3wgt}.}
Suppose the above assumption on the values of $\eta_{i,t}$ holds. Taking the values of the $n$-component weight vector $\vec{w}_t$ at time $t$, by Proposition~\ref{prop:fourterm},
\begin{eqnarray*}
r_t&=&\sum_{i^+\in I^+}w_{i^+,t-1}\frac{1}{1+r_{t-1}}-\sum_{j^+\in J^+}w_{j^+,t-1}\frac{1}{1+r_{t-1}}\\
&+&\sum_{i^-\in I^-}w_{i^-,t-1}\frac{1}{1-r_{t-1}}-\sum_{j^-\in J^-}w_{j^-,t-1}\frac{1}{1-r_{t-1}}.
\end{eqnarray*}
Then by assumption we have $J^-=\emptyset$. This is equivalent to
\begin{eqnarray*}
r_t&=&\sum_{i^+\in I^+}w_{i^+,t-1}\frac{1}{1+r_{t-1}}-\sum_{j^+\in J^+}w_{j^+,t-1}\frac{1}{1+r_{t-1}}\\
&+&\sum_{i^-\in I^-}w_{i^-,t-1}\frac{1}{1-r_{t-1}}-0.
\end{eqnarray*}
From Eqn.~\ref{eq:edgealt} we know that we can write the edge value at iteration $t$ as $r_t=1-2\sum_{j\in J}w_{j,t}$. This give us
\[r_t=1-2\sum_{j^+\in J^+}\frac{w_{j^+,t-1}}{1+r_{t-1}}\]
given that $J^-=\emptyset$. Thus upon rewriting this equation with the weight update formula, we also have that
\[r_t=\frac{1+r_{t-1}-2\sum_{j^+\in J^+}w_{j^+,t-1}}{1+r_{t-1}}.\]
 Conversely, if we let
\[r_t=\frac{1+r_{t-1}-2\sum_{j^+\in J^+}w_{j^+,t-1}}{1+r_{t-1}}\]
then
\[r_t=1-\frac{2\sum_{j^+\in J^+}w_{j^+,t-1}}{1+r_{t-1}}.\]
Rearranging slightly we find that
\[r_t=1-2\sum_{j^+\in J^+}\frac{w_{j^+,t-1}}{1+r_{t-1}}\]
which is simply our original definition for the edge with $J^-=\emptyset$.
\end{proof}


\hypertarget{20}{}
\begin{proof}{\bf of Proposition~\ref{prop:wvals}.}
We already know by Eqn.~\ref{eq:edgealt} that
\[\sum_{j^+\in J^+}w_{j^+,t-1}\frac{1}{1+r_{t-1}}=\frac{1-r_t}{2}\]
implied by our definition of the edge value. So by the proof of Theorem~\ref{thm:3wgt} 
\begin{eqnarray*}
\frac{1+r_t}{2}&=&\sum_{i^+\in I^+}w_{i^+,t-1}\frac{1}{1+r_{t-1}}+\sum_{i^-\in I^-}w_{i^-,t-1}\frac{1}{1-r_{t-1}}.
\end{eqnarray*}
Since $I^-$ indexes each of the weights that correspond to a misclassified data point we know that
\[\sum_{i^-\in I^-}w_{i^-,t-1}=\frac{1-r_{t-1}}{2}\]
so that
\[\frac{1+r_t}{2}=\sum_{i^+\in I^+}w_{i^+,t-1}\frac{1}{1+r_{t-1}}+\frac{1}{2}\]
and hence the result.
\end{proof}

\hypertarget{23}{}
\begin{proof}{\bf of Theorem~\ref{thm:gencyc}.}
Assume the conditions given above. Now let $\{\beta_{i^+,t}\}_{i^+\in I^+}$, $\{\beta_{i^-,t}\}_{i^-\in I^-}$, $\{\beta_{j^+,t}\}_{j^+\in J^+}$ be contributions at iteration $t$ as defined in Definition~\ref{def:altw}. For the $n$-component weight vector $\vec{w}_t$ and cycle $\{r_l,...,r_{l-1+k}\}$ we assume two separate lattices exist from iterations $t+sk$ to $t+ck$ for $s<c$, $l\leq t$. We call them lattices $A,B$.

Fix $q\in\mathbb{N}$ such that $sk\leq q\leq ck$. Assume lattices $A,B$ agree in terms of mistake dichotomy and corresponding weights at iteration $t+q$. We know that the edge values are already fixed by the cycle, so we want to show that these edge values and the common iteration determines both $A,B$. Suppose by way of contradiction that $A,B$ differ at iteration $t+q+1$ and let $\eta^A_{t+q},\eta^B_{t+q}$ be the mistake dichotomies of $A,B$ respectively for iteration $t+q$.

We know that $\eta^A_{t+q}=\eta^B_{t+q}$ by assumption and that the edge values are fixed. Given this the weight vector $\vec{w}_{t+q+1}$ agrees across both lattices after updating from the previous weights. Using our alternative weights we may assume that for $i_A\in I_A\cup J_A$ for lattice $A$
\[r_{t+q+1}=\sum_{i^+_A\in I^+_A}\beta_{i^+_A,t+q+1}\frac{r_{t+q+1}}{2}+\sum_{i^-_A\in I^-_A}\beta_{i^-_A,t+q+1}\frac{1}{2}-\sum_{j^+_A\in J^+_A}\beta_{j^+_A,t+q+1}\frac{1-r_{t+q+1}}{2}.\]
Given that $\eta^A_{t+q+1}\neq\eta^B_{t+q+1}$, the two dichotomies differ by $d\in\mathbb{N}$ components. Recall that we have assumed $\nabla$ as part of the setup of this proof and fix $i^-\in I^-$. Our assumption means that $A,B$ may not differ by setting $\beta_{i^-_A,t+q+1}\frac{1}{2}\sim-1$. Otherwise both iterations $t+q,t+q+1$ correspond to misclassified data points on row $i^-$ of the lattice, which contradicts $\nabla$. This is because an alternative weight has factor $\frac{1}{2}$ if it correspondeds to a misclassification in the previous iteration.

Since switching positive terms does not change the sum above for $r_{t+q+1}$, our interest is in switching $I^+,J^+$ index terms. Given that $A,B$ both induce a sum as above, those sums must be equal. We will index the terms on which $A,B$ disagree with $i^+_d,j^+_d\in I_d\cup J_d\subset I\cup J$. With the terms for which $A,B$ agree we can simply cancel them in our equality of the two lattices' edge updates, leaving us with a useful relation
$$\sum_{i^+_d\in I^+_d}\beta_{i^+_d,t+q+1}\frac{r_{t+q+1}}{2}-\sum_{j^+_d\in J^+_d}\beta_{j^+_d,t+q+1}\frac{1-r_{t+q+1}}{2}$$
$$=$$
$$\sum_{j^+_d\in J^+_d}\beta_{j^+_d,t+q+1}\frac{1-r_{t+q+1}}{2}-\sum_{i^+_d\in I^+_d}\beta_{i^+_d,t+q+1}\frac{r_{t+q+1}}{2}
$$
which implies that
\[\sum_{i^+_d\in I^+_d}\beta_{i^+_d,t+q+1}r_{t+q+1}=\sum_{j^+_d\in J^+_d}\beta_{j^+_d,t+q+1}(1-r_{t+q+1})\]
Then we obtain
\[\frac{r_{t+q+1}}{1-r_{t+q+1}}=\frac{\sum_{j^+_d\in J^+_d}\beta_{j^+_d,t+q+1}}{\sum_{i^+_d\in I^+_d}\beta_{i^+_d,t+q+1}}\in\mathbb{Q}\]
given our assumption in Remark~\ref{rem:rat}. This means the reciprocal of both sides of the above equation are rational numbers, too. But that is a contradiction because the left-hand side or its reciprocal is a periodic value of the Farey map, which are always quadratic irrationals. Thus, iteration $t+q+1$ is identical across lattices $A,B$.

For iteration $t+q-1$ the argument is identical. Since we can use this argument for any $q$ so that $sk\leq q\leq ck$, the same holds for the entirety of both lattices. This means that $A=B$.
\end{proof}

\newpage

\section{Experimental results}\label{sec:appB}
This appendix is entirely dedicated to experimental results that show the occurrence of the periodic learning condition that induce our titular limit cycles. AdaBoost's limit cycles were first introduced using matrices of mistake dichotomies in \cite{rudin} without the use of natural data sets. In \cite{ortiz} the authors suggest that the cycling of AdaBoost is a result of carefully selected \emph{synthetic data sets} (these matrices of mistake dichotomies), and we find that this is not the case.

\begin{center}
\begin{tabular}{||c c c c||} 
 \hline
 Data set & Original size & Sample size & Tree $($depth, leaves$)$ \\ [0.5ex] 
 \hline\hline
 Iris & 150 & 150 & (3,4) \\ 
 \hline
 seeds & 210 & 200 & (3,5) \\
 \hline
 wine & 178 & 160 & (3,3) \\
 \hline
 Dry Bean & 13611 & 120 & (3,3) \\
 \hline
 Breast Cancer & 286 & 50 & (1,2) \\ [1ex] 
 \hline
\end{tabular}

The above table lists the data sets used in our experimental results. Due to unknown reasons relating to some problem points in the data sets, data points were at times sampled randomly.

\end{center}

While this condition $\nabla$ is not in the language of probability theory like much of machine learning, it is an interesting behavior of the AdaBoost algorithm that does take place in natural data sets. The deterministic nature of this algorithm lends itself to the case-by-case empirical results found in this appendix. Throughout this appendix the only probabilistic component is in the sampling of data points from certain data sets found on the UCI data set repository.

\begin{figure}[htp]
    \centering
    \includegraphics[width=11cm]{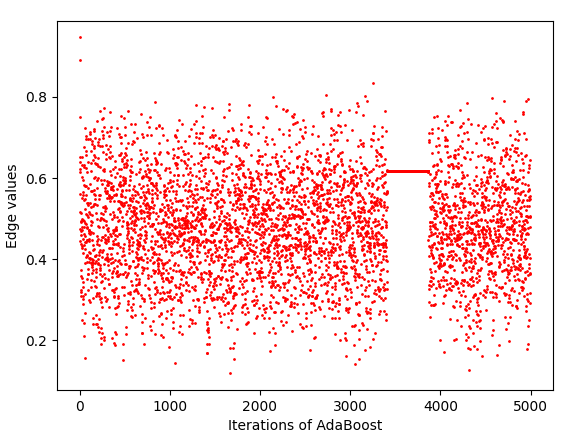}
    
    \emph{Iris Data Set} with no sampling of points. Each graph shown in this appendix was constructed in the same way.
    
\end{figure}

We can see in the image above that the UCI \emph{Iris} data set exhibits the $\frac{\sqrt{5}-1}{2}$ fixed point cycle from Section~3. The graph was created using edge values as the vertical axis and AdaBoost iterations the horizontal axis. This exact visualization may be replicated by graphing the edge values of the AdaBoost algorithm on this small data set using a maximum tree depth of $3$ and a maximum leaf count of $4$.

Given the result of Theorem~\ref{thm:3wgt}, AdaBoost cycles on this fixed point upon achieving the Farey map inverse edge update. Since this theorem is an \emph{if and only if} statement, we can conclude that AdaBoost may possibly cycle on this fixed point if and only if the periodic learning condition holds. Otherwise, the edge update would not take the Farey map inverse form that has $\frac{\sqrt{5}-1}{2}$ as a fixed point.

A possible reason for this cycling not being apparent to previous researchers is due to some concerns in finding appropriate data sets. For one, most small data sets from UCI, besides the \emph{Iris} data set as above, require a sampling of data points to exhibit the cycling behavior. As well, one must spend a serious amount of time looking at various data sets, and some data set samplings can be seen cycling only after tens of thousands of iterations.

\begin{figure}[htp]
    \centering
    \includegraphics[width=9cm]{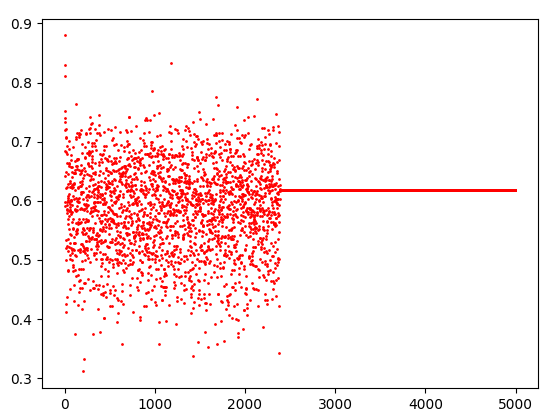}
    
    \emph{seeds Data Set} with $200$ sampled data points.
    
\end{figure}

The above edge value graph is for the UCI data set \emph{seeds Data Set} with $210$ data points. We were able to induce cycling on this data set by randomly sampling $200$ of the $210$ points using decision trees of $3$ depth and $5$ maximum leaf count. Interestingly, we do not get cycling when using the entire data set, suggesting that some small number of data points are the problem with respect to the non-occurrence of the $\nabla$ condition.

\begin{figure}[htp]
    \centering
    \includegraphics[width=9cm]{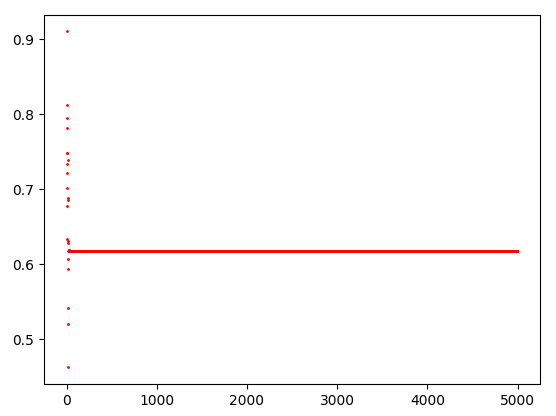}
    
    \emph{Wine Data Set} with $160$ sampled data points.
    
\end{figure}

\newpage

Our graph above is a generated from a random sampling of $160$ data points of the $178$ data point \emph{Wine Data Set} found on the UCI repository. Again, although this sampling of points is not much smaller than the entire data set, cycling is not induced by AdaBoost running on the original set. As before, the decision tree depth was $3$ and maximum leaf count $3$.

\begin{figure}[htp]
    \centering
    \includegraphics[width=9cm]{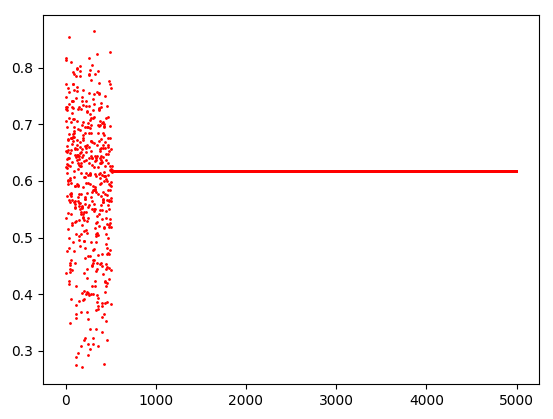}
    
    \emph{Dry Bean Data Set} with $120$ sampled data points.
    
\end{figure}

Another method for finding cycles in natural data sets was to take random samples of very large sets. Above is a graph for a $120$ point sample of the UCI \emph{Dry Bean Data Set}. The decision trees used had depth $3$ and maximum leaf count $3$. This set originally contains $13611$ data points, and many of these small random samplings exhibit cycling behavior.

\begin{figure}[htp]
    \centering
    \includegraphics[width=9cm]{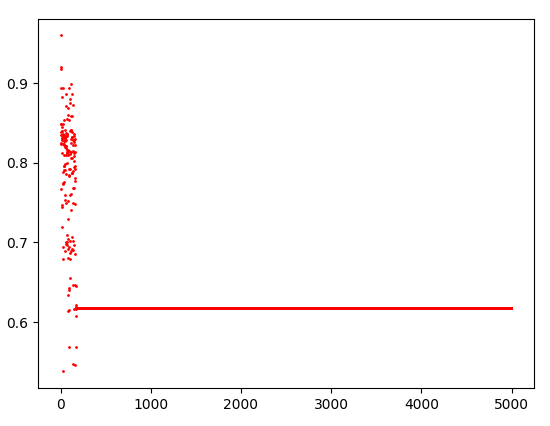}
    
    \emph{Breast Cancer Data Set} with $50$ sampled data points.
    
\end{figure}

Above is a graph of edge values from AdaBoost learning the UCI \emph{Breast Cancer Data Set} with random sampling of only $50$ data points, originally $286$. The learning properties of this data set are unusual compared to the rest shown in Appendix B, and only decision stumps were needed to exhibit cycling.

In \cite{ortiz} the writers conjecture that cycling does not take place on naturally occurring data. They use the \emph{Breast Cancer Data Set} learned using decision stumps among other data sets in an experimental result to show the unlikelihood of cycling behavior. Our experiments suggest that AdaBoost learning this data set in particular resists cycling unless a very small sampling is used. The reasons for this are unclear.

\end{document}